\documentclass{llncs}
\usepackage{graphicx}
\usepackage{multicol}
\usepackage{amsmath, amssymb, amscd, tikz}
\usepackage{makeidx}
\newtheorem{Def}{Definition}
\newtheorem{Th}{Theorem}
\newtheorem{Cor}{Corollary}[Th]

\begin{document}
%
%
\title{Exploiters-Based Knowledge Extraction in Object-Oriented Knowledge Representation}
\titlerunning{Inheritance in Object-Oriented Knowledge Representation}  
%
\author{Dmytro Terletskyi}
\authorrunning{} 
%
\tocauthor{}
\institute{Taras Shevchenko National University of Kyiv, Kyiv, 03680, Ukraine
\email{dmytro.terletskyi@gmail.com},\\
\texttt{http://cyb.univ.kiev.ua/en/departments.is.terletskyi.html}
}

\maketitle              

\begin{abstract}
This paper contains the consideration of knowledge extraction mechanisms of such object-oriented knowledge representation models as frames, object-oriented programming and object-oriented dynamic networks. In addition, conception of universal exploiters within object-oriented dynamic networks is also discussed. The main result of the paper is introduction of new exploiters-based knowledge extraction approach, which provides generation of a finite set of new classes of objects, based on the basic set of classes. The methods for calculation of quantity of new classes, which can be obtained using proposed approach, and of quantity of types, which each of them describes, are proposed. Proof that basic set of classes, extended according to proposed approach, together with union exploiter create upper semilattice is given. The approach always allows generating of finitely defined set of new classes of objects for any object-oriented dynamic network. A quantity of these classes can be precisely calculated before the generation. It allows saving of only basic set of classes in the knowledge base.

\end{abstract}
\section{Introduction}

Nowadays methods of knowledge extracting and reasoning about knowledge are significant constituent part of majority of knowledge-based systems. It gives an opportunity to extract or to obtain new knowledge, based on such called, \emph{basic knowledge}. Such abilities make knowledge-based systems intelligent and applicable ones at least in such areas of artificial intelligence as information search, problem solving, planning, patterns recognition, decision making, etc.

Currently there is variety of knowledge representation models (KRMs), which implement different approaches to knowledge representation. One of them is an object-oriented knowledge representation, the main idea of which is representation of knowledge in terms of objects, classes of objects and relationships among them. Nowadays the most famous KRMs within this approach are frames and object-oriented programming (OOP). Both of them have their own knowledge extraction methods, which give some abilities for new knowledge obtaining. Let us consider these mechanisms and their main features in more detail.

\section{Knowledge Extraction in Frames and OOP}

Frames as a KRM provide representation of knowledge in terms of hierarchies of frames (system of frames), where particular frame is a class-frame or instance-frame \cite{Brachman-Levesque}, \cite{Negnevitsky}. Each frame is connected with others via relations of generalization (\emph{is-a}, \emph{a-kind-of}, \emph{an-instance-of}, etc.), aggregation (\emph{a-part-of}, \emph{part-whole}, etc.) and association (\emph{owns}, \emph{plays}, \emph{creates}, etc.). Relation of generalization provides implementation of inheritance mechanism that allows more specific frames, that situated lower in the hierarchy, inherit all slots from more general frames. Such structure of the system allows efficient knowledge representation, because it is based on the idea of representation of new knowledge via previously represented ones.

In addition, each frame can have definite procedural attachments, which allow execution of actions on it. Some procedures execute only when they are in need (when-procedures), other ones execute in particular situations. Thus, extracting of new knowledge in the frame system can be done by dint of the reasoning in inheritance hierarchy or by means of procedural attachments executing \cite{Ueno}. However, frames support two kinds of inheritance -- single and multiple ones \cite{Negnevitsky}. Inheritance can cause such problems as \emph{problem of exceptions}, \emph{problem of redundancy} and \emph{problem of ambiguity} \cite{Al-Asady}, \cite{Touretzky}. Frames also allow overriding of values of slots in the instance-frames \cite{Negnevitsky}, that leads to the situation when the subclass or instance goes beyond its superclass.

In contrast to frames, OOP is divided on two styles -- \emph{class-based} and \emph{prototype based} ones \cite{Craig}. Similarly to frames, first approach provides knowledge representation in terms of hierarchies of classes, using inheritance. Second one gives an opportunity for knowledge representation in terms of prototypes. Despite that both styles are object-oriented ones, they have significant differences.

Class-based approach provides ability to work only with instantiated objects, to change values of their properties, to execute their methods and in such a way to obtain new knowledge. We cannot change the description of a class, type of object or hierarchy of classes during program execution. It means that we can obtain only objects of the same type with changed values of the properties. In addition, inheritance in OOP causes the same problems as in frames \cite{Craig}.

Prototype-based style gives an opportunity to operate with prototypes. Each new prototype is a modified clone of another one. It means that such approach is more flexible for description of new concepts, because it allows creating of new prototypes during program execution and implements the idea of partial inheritance. However, it requires much more computer memory and leads to redundancy in representation of particular prototypes.

\section{Object-Oriented Dynamic Networks}
Besides mentioned object-oriented KRMs, there is one more KRM, such as object-oriented dynamic networks (OODNs), that was proposed in \cite{Terletskyi-2}. This KRM has similarity with all mentioned KRMs, however it also has some specific features, which give new opportunities in knowledge representation within object-oriented approach. Let us consider structure of this model.
\begin{Def}
\label{oodn}
Object-Oriented Dynamic Network is a 5-tuple
\[OODN=(O,C,R,E,M),\]
where:
\begin{itemize}
 \item $O$ -- a set of objects;
 \item $C$ -- a set of classes of objects, which describe objects from the set $O$;
 \item $R$ -- a set of relations, which are defined on the set $O$ and $C$;
 \item $E$ -- a set of exploiters, which are defined on the set $O$ and $C$;
 \item $M$ -- a set of modifiers, which are defined on the set $O$ and $C$.
\end{itemize}
\end{Def}
Definitions of all elements from the tuple $OODN=(O,C,R,E,M)$ were introduced and considered in detail in \cite{Terletskyi-2}. Each object from the set $O$ has some properties, which define it as an essence. There are two kinds of object's properties -- \emph{quantitative} and \emph{qualitative} ones, which definitions were introduced in \cite{Terletskyi-1}. However, we need to consider the properties of a class of objects. Let us define them.
\begin{Def}
Quantitative property of class of objects $T$ is a tuple
\[p(T)=(v(p(T)),u(p(T))),\]
where $v(p(T))$ is an quantitative value of $p(T)$ and $u(p(T))$ are units of its measure.
\end{Def}
\begin{Def}
Qualitative property of class of objects $T$ is a verification function $p(T)=vf(T)$, which is defined as a mapping $vf(T):p(T)\rightarrow[0,1]$ and reflects the degree (measure) of truth (presence) of a property $p(T)$ for the class $T$.
\end{Def}
Let us define the conception of equivalence of these kinds of properties.
\begin{Def}
Two quantitative properties $p(T_1)$ and $p(T_2)$ of arbitrary classes of objects $T_1$ and $T_2$ are equivalent, i.e. $Eq(p(T_1),p(T_2))=1$, if and only if $(u(p(T_1))=u(p(T_2)))\wedge(v(p(T_1))=v(p(T_2)))$.
\end{Def}
\begin{Def}
Two qualitative properties $p_1(T_1)$ and $p_2(T_2)$ of arbitrary classes of objects $T_1$ and $T_2$ are equivalent, i.e. $Eq(p_1(T_1),p_2(T_2))=1$, if and only if $(vf_1(T_1)=vf_2(T_1))\wedge(vf_1(T_2)=vf_2(T_2))$.
\end{Def}
For every object of a class we can define methods, which can be applied to them and allow definition of their behaviour and manipulating on them.
\begin{Def}
Method of class of objects $T$ is a function $f(T)$, which can be applied to the class $T$, considering the features of its specification (vector of properties).
\end{Def}
From the previous definition, we can see that method is a function, which is defined under the properties. To define the equivalence of methods we should define the equivalence of two arbitrary functions, but in general case such problem is unsolvable one. So, we are going to introduce the equivalence of methods via following definition.
\begin{Def}
Two methods $f_1(T_1)$ and $f_2(T_2)$ of arbitrary classes of objects $T_1$ and $T_2$ are equivalent, i.e. $Eq(f_1(T_1),f_2(T_2))=1$, if and only if $(f_1(T_1)=f_1(T_2))\wedge(f_2(T_1)=f_2(T_2))$.
\end{Def}
It introduces the equivalence of two methods on the same argument. It means that two methods $f_1(T_1)$ and $f_2(T_2)$ can be different as functions, however they can return the same results on the same objects.

Concepts of objects, classes and relations among them have different implementations in various KRMs. One of the main differences is the definition of the class. Within frames and OOP, concept of class is defined as abstract description of some quantity of objects, which have the same nature \cite{Brachman-Levesque}, \cite{Craig}. That is why, it is possible to conclude, that such class is a homogeneous one, because it contains only objects of the same type. Nevertheless, there are classes, which are inhomogeneous ones \cite{Terletskyi-1}. Within OODNs, there are two definitions for both types of classes. Let us consider them in more details.
\begin{Def}
\label{hc}
Homogeneous class of objects $T$ is a tuple $T=(P(T),F(T))$, where $P(T)$ is specification (a vector of properties) of some quantity of objects, and $F(T)$ is their signature (a vector of methods).
\end{Def}
According to this definition, all objects of such class have the same type, i.e. they have the same properties and methods as their class. Let us consider the definition of inhomogeneous class of objects.
\begin{Def}
\label{ihc}
Inhomogeneous (heterogeneous) class of objects $T$ is a tuple
\[T=(Core(T),pr_1(A_1),\dots,pr_n(A_n)),\]
where $Core(T)=(P(T),F(T))$ is the core of class of objects $T$, which includes only properties and methods similar to corresponding properties of specifications $P(A_1),\dots,P(A_n)$ and corresponding methods of signatures $F(A_1),\dots,F(A_n)$ respectively, and where $pr_i(A_i)=(P(A_i),F(A_i))$ , $i=\overline{1,n}$ are projections of objects $A_1,\dots,A_n$, which consist of properties and methods typical only for these objects.
\end{Def}
According to this definition, it is possible to represent certain amount of any types by dint of one class within object-oriented approach. While representation of each type of objects in OOP always requires definition of new class.

Analyzing definitions \ref{hc} and \ref{ihc}, we can conclude that a homogeneous class of objects defines a type of objects. In this case the type and the class of objects mean the same. However, an inhomogeneous class of objects defines at least two different types of objects within one class of objects that is why in this case the type and the class are not equivalent. In other words inhomogeneous class of objects includes a few types of objects. Let us define a type of inhomogeneous class of objects.
\begin{Def}
\label{type}
Type of arbitrary inhomogeneous class of objects
\[T=(Core(T),pr_1(T),\dots,pr_n(T))\]
is a homogeneous class of objects $T_i=(Core(T),pr_i(T))$, where $i=\overline{1,n}$.
\end{Def}

Now, let us define the following tree kinds of subclass relations for classes of objects: homogeneous $\subseteq$ homogeneous, inhomogeneous $\subseteq$ inhomogeneous and homogeneous $\subseteq$ inhomogeneous.
\begin{Def}
\label{hh-sc}
Homogeneous class of objects $T_1=(P(T_1),F(T_1))$ is a subclass of homogeneous class of objects $T_2=(P(T_2),F(T_2))$, i.e. $T_1\subseteq T_2$ if and only if
\begin{gather*}
(\forall p_i\in P_1\ \exists p_j\in P_2\ |\ Eq(p_i,p_j)=1)\wedge(\forall f_k\in F_1\ \exists f_w\in F_2\ |\ Eq(f_k,f_w)=1),
\end{gather*}
where $P_1$, $P_2$, $F_1$, $F_2$ are sets, which contain elements of vectors $P(T_1)$, $P(T_2)$, $F(T_1)$, $F(T_2)$ respectively and $i=\overline{1,|P_1|}$, $j=\overline{1,|P_2|}$, $k=\overline{1,|F_1|}$, $w=\overline{1,|F_2|}$.
\end{Def}
\begin{Def}
Inhomogeneous class of objects
\[T_1=(Core(T_1),pr_1(T_1),\dots,pr_n(T_1))\]
is a subclass of inhomogeneous class of objects
\[T_2=(Core(T_2),pr_1(T_2),\dots,pr_m(T_2)),\]
i.e. $T_1\subseteq T_2$ if and only if
\begin{gather*}
(\forall a_i\in C_1\ \exists a_j\in C_2\ |\ Eq(a_i,a_j)=1)\wedge\\
\wedge(\forall b_{h_k}\in pr_h\ \exists!b_{v_w},pr_v\wedge b_{v_w}\in pr_v\ |\ Eq(b_{h_k},b_{v_w})=1),
\end{gather*}
where $C_1$, $C_2$, $pr_h$, $pr_v$ are sets, which contain elements of vectors from the sets $Core(T_1)$, $Core(T_2)$, $pr_h(T_1)$, $pr_v(T_2)$ respectively and $i=\overline{1,|C_1|}$, $j=\overline{1,|C_2|}$, $k=\overline{1,|pr_h|}$, $w=\overline{1,|pr_v|}$, $h=\overline{1,n}$, $v=\overline{1,m}$.
\end{Def}
\begin{Def}
\label{hi-sc}
Homogeneous class of objects $T_1=(P(T_1),F(T_1))$ is a subclass of inhomogeneous class of objects $T_2=(Core(T_2),pr_1(T_2),\dots,pr_n(T_2))$, i.e. $T_1\subseteq T_2$ if and only if
\begin{gather*}
(\forall p_i\in P_1\ \exists p_j\in C_2\vee pr_v\ |\ Eq(p_i,p_j)=1)\wedge\\
\wedge(\forall f_k\in F_1\ \exists f_w\in C_2\vee pr_v\ |\ Eq(f_k,f_w)=1),
\end{gather*}
where $P_1$, $F_1$ are sets, which contain elements of vectors $P(T_1)$, $F(T_1)$ and $C_2$, $pr_v$ are sets, which contain elements of vectors from the sets $Core(T_2)$ and $pr_v(T_2)$ respectively, $i=\overline{1,|P_1|}$, $j=\overline{1,|C_2|+|pr_v|}$, $k=\overline{1,|F_1|}$, $w=\overline{1,|C_2|+|pr_v|}$, $v=\overline{1,n}$.
\end{Def}

According to the definitions of class of objects, it is possible to define the vector of methods for each class of objects, concerning its specification. Such kind of methods are \emph{internal} ones, because they are defined under particular properties of the class. Besides them, there are methods, which are called \emph{external} ones and are defined under whole specification of the class. Depending on the character of actions, all methods can be divided on \emph{exploiters} and \emph{modifiers}. Exploiters do not change objects or classes, they just use them as parameters for new knowledge obtaining. While, modifiers change the basic knowledge and allow modelling of their changes or evolution over the time. That is why $F(T)$ contains internal, $E$ and $M$ contain external methods of the class of objects.

Summarizing, OODN can be considered as two conceptual parts. First of them is declarative, which includes sets $O$, $C$, $R$, and allows representation of knowledge about particular domain. Second part is procedural one. It includes sets $E$, $M$ and provides the tools for obtaining new knowledge from basic ones. All following considerations are connected with applications of procedural part of OODN, in particular exploiters-based knowledge extraction.

\section{Exploiters-Based Knowledge Extraction}

As it was mentioned above, exploiters form significant constituent of procedural part of OODN. Generally, we can define variety of exploiters for each class of objects, however majority of them are locally closed under their classes. That is why, such universal exploiters as \emph{union}, \emph{intersection}, \emph{difference} and \emph{symmetrical difference} were introduced in \cite{Terletskyi-1}. Their applications allow building of new classes of objects. This fact has significant value not only in knowledge extraction, but also in programming, because it is a step toward the implementation of runtime class generation.

Let us define union exploiter for classes of objects, using definition \ref{type}.
\begin{Def}
\label{union}
Union $\cup$ of two arbitrary nonequivalent classes of objects $T_1$ and $T_2$ is an inhomogeneous class of objects $T=(Core(T),pr_1(T),\dots,pr_n(T))$, where $Core(T)=(P(T),F(T))$ is its core and includes only properties and methods, which are similar for types $T_{1_1},\dots,T_{1_m},T_{2_1},\dots,T_{2_k}$, and where $pr_j(T)=(P(T),F(T))$ is projection of type $T_{i_j}$, $i=\overline{1,2}$, $j=\overline{1,n}$, $n=m+k$ which consist of properties and methods typical only for this type.
\end{Def}
Application of union exploiter to classes of objects has some important features besides generation of new classes of objects. Let us formulate and prove a few theorems, which illustrate these features.
\begin{Th}
\label{th-1}
For any $OODN=(O,C=\{T_1,\dots,T_n\},R,E=\{\cup\},M)$, where classes $T_1,\dots,T_n$ are homogeneous ones and do not have any common properties and methods, all possible applications of union exploiter $\cup$, including all possible its superpositions, to classes of objects from the set $C$ always generate finite quantity of new classes of objects, which can be calculated by the following formula:
\[q(C_E)=2^n-n-1,\]
where $n=|C|$.
\end{Th}
\begin{proof}
It is known that the number of all possible unique $k$-elements combinations from the $n$-elements set can be calculated as $C_n^k$. Similarly, the number of all possible unique classes of objects created from the basic set of classes $C=\{T_1,\dots,T_n\}$ using union exploiter can be represented as a combination of $k=\overline{2,n}$ different classes from the set $C$. It is known that
\[\sum_{n=0}^kC_n^k=2^n.\]
However, we cannot create classes of objects, which describe 1 and 0 different types, applying union exploiter to the classes of objects from the set $C$, i.e. we do not count $C_n^0$ and $C_n^1$, we can conclude that
\[q(C_E)=\sum_{k=0}^nC_n^k-C_n^0-C_n^1=\sum_{k=2}^{n}C_n^k=2^n-n-1.\]
$\square$
\end{proof}
Using Theorem \ref{th-1}, we can formulate one more important theorem.
\begin{Th}
\label{th-2}
Set of classes of objects
\[C=\{T_1,\dots,T_n,T_{n+1},\dots,T_{2^n-1}\}\]
of any OODN, extended according to Theorem~\ref{th-1}, with exploiter $\cup$ create the upper semilattice, where class $T_{1,\dots,n}=T_1\cup\dots\cup T_n$ is its greatest upper bound.
\end{Th}
\begin{proof}
According to the definition of upper semilattice, it should be a system $SL=(A,\Omega)$, where $A$ is a poset, $\Omega=\{\vee\}$ and $\vee$ is binary, idempotent, commutative and associative operation \cite{Birkhoff}, \cite{Davey-Priestley}.

In our case, carrier of upper semilattice is the set of classes $C$, where we define exploiter $\cup$, thus $SL=(C,\Omega)$, where $\Omega=\{\cup\}$. From the definition~\ref{union} it follows, that mentioned properties of $\vee$ are true for $\cup$, i.e.
\begin{enumerate}
\item $T_1\cup T_1=T_1$;
\item $T_1\cup T_2=T_2\cup T_1$;
\item $T_1\cup(T_2\cup T_3)=(T_1\cup T_2)\cup T_3$.
\end{enumerate}
Now, let us show that $C$ is a poset. For this, we should define $\forall T_1,T_2\in C\ | T_1\subseteq T_2\Leftrightarrow T_1\cup T_2=T_2$ and show that $\subseteq$ is a relation of partial order under the set $C$. It means, we should prove that relation $\subseteq$ is reflexive, antisymmetric, and transitive one.
\begin{enumerate}
\item $T_1\subseteq T_1\Leftrightarrow T_1\cup T_1=T_1$ follows from idempotency of $\cup$;
\item $T_1\subseteq T_2\Leftrightarrow T_1\cup T_2=T_2$, $T_2\subseteq T_1\Leftrightarrow T_2\cup T_1=T_1$ and from commutativity of $\cup$, we can conclude that $T_1=T_2$;
\item $T_1\subseteq T_2\Leftrightarrow T_1\cup T_2=T_2$, $T_2\subseteq T_3\Leftrightarrow T_2\cup T_3=T_3\Rightarrow(T_1\cup T_2)\cup T_3=T_1\cup(T_2\cup T_3)=T_1\cup T_3=T_3\Rightarrow T_1\cup T_3=T_3\Leftrightarrow T_1\subseteq T_3$.
\end{enumerate}
Thus, $SL=(C=\{T_1,\dots,T_n,T_{n+1},\dots,T_{2^n-1}\},\Omega=\{\cup\})$ is an upper semilattice, where $T_{1,\dots,n}=T_1\cup\dots\cup T_n$ is its greatest upper bound. $\square$
\end{proof}
Using results of Theorem \ref{th-2}, we can formulate the following corollary.
\begin{Cor}
\label{cor-31}
Set of classes of objects $C$ of any OODN, extended according to Theorem~\ref{th-1}, and union exploiter $\cup$, which is defined under it, create a finitely-generated universal algebra
\[G=\left(C=\{T_1,\dots,T_n,T_{n+1},\dots,T_{2^n-1}\},\Omega=\{\cup\}\right),\]
where $C_b=\{T_1,\dots,T_n\}$ is generative set for the set $C$.
\end{Cor}

Now let us consider an example, which illustrates specific of exploiters-based knowledge extraction within OODN. Let us define the OODN
\[Salad=(O,C,R,E,M),\]
which describes some ingredients of a salad, for example cucumber, tomato, onion, cabbage, salt and sunflower oil. For this purpose, we define following sets of objects $O$, classes of objects $C$ and set of relations $R$
\begin{gather*}
O=\{cuc,tom,on,cab,sal,soil\},\\
C=\{Cuc,Tom,Cab,On,Spi,Oil\},\\
R=\{cuc\xrightarrow{an-inst.-of}Cuc,\ tom\xrightarrow{an-inst.-of}Tom,\ cab\xrightarrow{an-inst.-of}Cab,\\ on\xrightarrow{an-inst.-of}On,\ sal\xrightarrow{an-inst.-of}Spi,\ soil\xrightarrow{an-inst.-of}Oil.\}.
\end{gather*}
Suppose the set of exploiters is defined as $E=\{\cup\}$. We do not define the set of modifiers $M$, because it is not necessary within consideration of exploiters-based knowledge extraction.

Let us define the specifications of classes from set $C$ in the following way
\begin{gather*}
P(Cuc)=(p_1(Cuc),\dots,p_4(Cuc)),\ P(Tom)=(p_1(Tom),\dots,p_4(Tom)),\\
P(Cab)=(p_1(Cab),\dots,p_4(Cab)),\ P(On)=(p_1(On),\dots,p_4(On)),\\
P(Spi)=(p_1(Spi),\dots,p_4(Spi)),\ P(Oil)=(p_1(Oil),\dots,p_4(Oil)),
\end{gather*}
where $p_1(Cuc)$, $p_1(Tom)$, $p_1(Cab)$, $p_1(On)$ -- masses of vegetables, $p_1(Spi)$ -- type of spices, $p_1(Oil)$ -- type of oil, $p_2(Cuc)$, $p_2(Tom)$, $p_2(Cab)$, $p_2(On)$ -- colors of vegetables, $p_2(Spi)$ -- mass of spices, $p_2(Oil)$ -- volume of oil, $p_3(Cuc)$, $p_3(Tom)$, $p_3(Cab)$, $p_3(On)$ -- freshness of vegetables, $p_3(Spi)$ -- taste of spices, $p_3(Oil)$ -- color of oil, $p_4(Cuc)$, $p_4(Tom)$, $p_4(Cab)$, $p_4(On)$, $p_4(Spi)$, $p_4(Oil)$ -- prices. Values of all properties of these classes are defined in Table~\ref{tab-1}.

Let us define the specifications of objects from the set $O$, using specifications of their classes (see Table~\ref{tab-2}).

\begin{table}
\centering{
\caption{Specifications of classes $Cuc$, $Tom$, $Cab$, $On$, $Spi$, $Oil$}
\label{tab-1}
\begin{tabular}{lcccccc}
\hline\noalign{\smallskip}
\multicolumn{1}{c}{$p_i$} & Cuc & Tom & Cab & On & Spi & Oil\\
\noalign{\smallskip}
\hline
\noalign{\smallskip}
$p_1$ & $[0.07,0.18]$ kg & $[0.08,0.2]$ kg & $[0.4,1.3]$ kg & $[0.05,0.1]$ kg & undefined & undefined\\
$p_2$ & green & red & green & green-white & $[0.1,1.0]$ kg & $[0.5,1.0]$ l\\
$p_3$ & undefined & undefined & undefined & undefined & undefined & yellow\\
$p_4$ & $3$ USD/kg & $3.5$ USD/kg & $4$ USD/kg & $2$ USD/kg & $12$ USD/kg & $9$ USD/l\\
\hline
\end{tabular}}
\end{table}

\begin{table}
\centering{
\caption{Specifications of objects $cuc1$, $cuc2$, $tom1$, $tom2$, $cab$, $on$, $sal$, $soil$}
\label{tab-2}
\begin{tabular}{lcccccccc}
\hline\noalign{\smallskip}
\multicolumn{1}{c}{$p_i$} & cuc1 & cuc2 & tom1 & tom2 & cab & on & sal & soil\\
\noalign{\smallskip}
\hline
\noalign{\smallskip}
$p_1$ & $0.09$ kg & $0.08$ kg & $0.12$ kg & $0.1$ kg & $0.5$ kg & $0.1$ kg & salt & sunflower\\
$p_2$ & green & green & red & red & green & green-white & $0.5$ kg & $0.5$ l\\
$p_3$ & $1$ & $1$ & $1$ & $1$ & $1$ & $1$ & salty & yellow\\
$p_4$ & $0.27$ USD & $0.24$ USD & $0.42$ USD & $0.35$ USD & $2$ USD & $0.2$ USD & $6$ USD & $4.5$ USD\\
\hline
\end{tabular}}
\end{table}

We have described the OODN for the salad. Clearly that all elements of sets $O$, $C$ and $R$ are basic knowledge. Let us obtain all possible new knowledge from them using exploiter $\cup$. According to Theorems~\ref{th-1}-\ref{th-2} we obtain such 15 classes, that each of them describes 2 different types of objects
\begin{gather*}
CucTom,\ CucCab,\ CucOn,\ CucSpi,\ CucOil,\ TomCab,\ TomOn,\ TomSpi,\\
TomOil,\ CabOn,\ CabSpi,\ CabOil,\ OnSpi,\ OnOil,\ SpiOil;
\end{gather*}
such 20 classes, that each of them describes 3 different types of objects
\begin{gather*}
CucTomCab,\ CucTomOn,\ CucTomSpi,\ CucTomOil,\ CucCabOn,\\
CucCabSpi,\ CucCabOil,\ CucOnSpi,\ CucOnOil,\ CucSpiOil,\\
TomCabOn,\ TomCabSpi,\ TomCabOil,\ TomOnSpi,\ TomOnOil,\\
TomSpiOil,\ CabOnSpi,\ CabOnOil,\ CabSpiOil,\ OnSpiOil;
\end{gather*}
such 15 classes, that each of them describes 4 different types of objects
\begin{gather*}
CucTomCabOn,\ CucTomCabSpi,\ CucTomCabOil,\ CucTomOnSpi,\\
CucTomOnOil,\ CucTomSpiOil,\ CucCabOnSpi,\ CucCabOnOil,\\
CucCabSpiOil,\ CucOnSpiOil,\ TomCabOnSpi,\ TomCabOnOil,\\
TomCabSpiOil,\ TomOnSpiOil,\ CabOnSpiOil;
\end{gather*}
such 6 classes, that each of them describes 5 different types of objects
\begin{gather*}
CucTomCabOnSpi,\ CucTomCabOnOil,\ CucTomCabSpiOil,\\
CucTomOnSpiOil,\ CucCabOnSpiOil,\ TomCabOnSpiOil;
\end{gather*}
and 1 class, which describes 6 different types of objects
\[CucTomCabOnSpiOil.\]
As we can see, we obtain 57 new classes of objects, or in other words, 57 different combinations of salad's ingredients from the 6 basic ones. In such a way we extended the set of classes $C$ by adding new knowledge, extracted from basic ones. The most general obtained class $CucTomCabOnSpiOil$ is an inhomogeneous one and has following structure
\begin{gather*}
CucTomCabOnSpiOil=(pr_1(CucTomCabOnSpiOil),\dots,\\
pr_6(CucTomCabOnSpiOil)),
\end{gather*}
where
\begin{gather*}
pr_1(CucTomCabOnSpiOil)=P(Cuc), pr_2(CucTomCabOnSpiOil)=P(Tom),\\
pr_3(CucTomCabOnSpiOil)=P(Cab), pr_4(CucTomCabOnSpiOil)=P(On),\\
pr_5(CucTomCabOnSpiOil)=P(Spi), pr_6(CucTomCabOnSpiOil)=P(Oil).
\end{gather*}
All other obtained classes have the similar structure.

According to Theorem~\ref{th-2}, the extended set of classes $C$ and exploiter $\cup$ create upper semilattice. It means, that there is partial order relation $\subseteq$, which is defined on the set $C$. Furthermore, according to Corollary~\ref{cor-31} they create a finitely-generated universal algebra
\[G=(C=\{Cuc,Tom,Cab,On,Spi,Oil,\dots,CucTomCabOnSpiOil\},E=\{\cup\}),\]
where $C_b=\{Cuc,Tom,Cab,On,Spi,Oil\}\subseteq C$ is generative or basic set for the set $C$.

Summarizing, we obtained all possible unions of the basic classes from the set $C$. All these classes can be viewed as schemas or recipes for which we can use objects defined in Table~\ref{tab-2}. In means that in such a way we can create particular salad, moreover we can create different salads using one scheme putting different proportions of ingredients. Using chosen scheme, we can calculate different properties of the cooked salad, for example its prise, mass, etc.

\section{Conclusions}

This paper contains consideration of main features of knowledge extraction mechanisms of such object-oriented KRMs as frames, OOP and OODNs. Furthermore, conception of universal exploiters within object-oriented dynamic networks is also discussed.

The main achievement of the paper is introduction of new exploiters-based knowledge extraction method for OODNs, which always provides generating of finitely defined set of new classes of objects, based on the basic set of classes. The main features of the proposed method are:
\begin{itemize}
\item ability to calculate:
  \begin{itemize}
  \item quantity of new classes, which can be obtained, using proposed approach,
  \item quantity of different types, which each of obtained classes describes;
\end{itemize}
\item the basic set of classes of any OODN, extended according to proposed approach, together with union exploiter, create:
  \begin{itemize}
  \item upper semilattice,
  \item finitely generated universal algebra, for which the basic set of classes of OODN is a generative set.
  \end{itemize}
\end{itemize}
It allows us to extract new knowledge from the basic ones when we need them and to save only basic set of classes in the knowledge base and database. Moreover, obtained knowledge always have the defined structure, i.e. they form the upper semilattice. It means that we can use the results of upper semilattice theory in such kind of knowledge extraction and representation.

However, despite all advantages, proposed approach requires further research, at least in the following directions:
\begin{itemize}
\item study of the case when the basic set of classes of OODN contains classes that has some common properties or methods,
\item study of the case when the OODN is a fuzzy one,
\item adapting and usage of proposed approach in other known object-oriented KRMs.
\end{itemize}

%
%

%
%

\end{document}